\documentclass[journal]{IEEEtran}
\pdfoutput=1

\usepackage{cite}

\usepackage{algorithmic}

\usepackage{eqparbox} 
\usepackage[tight,footnotesize]{subfigure}
\usepackage[caption=false,font=footnotesize]{subfig}
\usepackage{fixltx2e}  
\usepackage{stfloats} 
\usepackage{url}

\usepackage[cmex10]{amsmath}
\usepackage{amssymb}

\usepackage{amsthm}    

\usepackage{bm}
\usepackage{array}

\usepackage{graphicx}     

\usepackage{epstopdf}
  
\usepackage{comment}     
\usepackage{pst-all}   
\usepackage{colortbl}  
\usepackage{tabularx}  
\usepackage{multirow}   
\usepackage{algorithm}

\usepackage[english]{babel}
\usepackage[latin5]{inputenc}

\newcommand{\E}[1]{{E\left[ #1 \right]} }  
\newcommand{\partiald}{\;d}

\newtheorem{lem}{Lemma}
\newtheorem{cor}{Corollary}

\newtheorem{defn}{Definition}
 
\newtheorem{exam}{Example}

\newtheorem{rem}{Remark}

\newcounter{foox}
\newcommand{\Note}[1]{{ \bigskip \color{blue}  \hrule TODO \arabic{foox}: #1 
\stepcounter{foox}
\hrule\color{black} \bigskip} }

\newcommand{\LL}[2]{\mathcal{L}_{#1}(#2)  }  % #1: const, #2: variable
\newcommand{\HE}[2]{\mathcal{H}_{#1}(#2)  }  % #1: const, #2: variable

\newcommand{\Alpha}{alpha }
\newcommand{\Beta}{beta }

\begin{document}

\title{Generalized Beta Divergence}

%\author{Y.K.~Yılmaz, ~\IEEEmembership{Member,~IEEE}
%\author{Y. Kenan.~Yılmaz}

\author{Y. Kenan~Yılmaz   ~\IEEEmembership{kenan@sibnet.com.tr}     
\thanks{This work is dedicated to
Gezi Park resisters in Istanbul's Taksim Square, who have been protesting for days against authoritarianism. 
}
}

% The paper headers
%\markboth{ArVix,~Vol.~03, No.~22, AAAA~2012, \today}%
%{Shell \MakeLowercase{\textit{et al.}}: Bare Demo of IEEEtran.cls for Journals}

\maketitle
\IEEEpeerreviewmaketitle
   
\begin{abstract} 
%\boldmath
This paper generalizes beta divergence beyond its classical form associated with power variance functions of Tweedie models. Generalized form is represented by a compact definite integral as a function of variance function of the exponential dispersion model. This  compact integral form simplifies derivations of many properties such as scaling, translation and expectation of the beta divergence. Further, we show that beta divergence and (half of) the statistical deviance are equivalent measures.

\end{abstract} 


\DeclareGraphicsExtensions{.eps,.jpg}

\section{Introduction}

Divergences and distributions are deeply related concepts studied extensively in various fields. This paper is another attempt that casts their relations specifically into that of beta divergences and dispersion models and studies  accordingly. The main consequence of this study is that beta divergence and (half of) statistical deviance are represented identical equations and they are, therefore,  equivalent measures. In this respect, formulation of beta divergence is generalized and thus is extended beyond its Tweedie related classical forms \cite{basu1998},\cite{eguchi2001},\cite{cichocki2010},\cite{cichocki11}  and is aligned with exponential dispersion models.  This is achieved  by defining beta divergences as a function of so-called {\it variance functions} of exponential dispersion models. One  immediate consequence is that we can compute beta divergence for non-Tweedie models such as negative binomial or hyperbolic secant distribution. Another consequence is compact integral representation of beta divergence  
 \begin{align*}
  d_\beta(x,\mu) &=  \int_\mu^x  \frac{x-t}{v(t)}    \partiald t  
\end{align*}
where $v(t)$ is  variance functions of the dispersion models.
This form gives a simple and intuitive way for statistical interpretation of beta divergence  by decomposing it  into the difference of two integrals 
 \begin{align*}
  d_\beta(x,\mu) &=  \int_\mu^x  \frac{x-t}{v(t)}    \partiald t = 
 \underbrace{ \int_{\mu_0}^x  \frac{x-t}{v(t)}    \partiald t}_{\LL{x}{x}}  - 
 \underbrace{\int_{\mu_0}^\mu  \frac{x-t}{v(t)}    \partiald t}_{\LL{x}{\mu}} 
\end{align*}
which is equal to the log-likelihood ratio of the full model to the parametric model
\begin{align*}
   \LL{x}{x} - \LL{x}{\mu} = \frac{1}{2} d_\nu (x,\mu)
\end{align*}
and that is half of the unit deviance $d_\nu$ by definition. This way, beta divergence is  linked to half of statistical deviance. Interestingly, quasi-log-likelihood, a deeply related concept to deviance, is defined by an identical integral form by Wedderburn in 1974 \cite{wedderburn1974} as (adapted notation)
 
%\begin{align*}
%   \frac{ \partial K(x,\mu)  }{ \partial \mu} = 
%     \frac{x - \mu}{v(\mu)}
%\end{align*}
%or equivalently
%
\begin{align*}
    K(x,\mu) = \int^{\mu} 
     \frac{x - t}{v(t)} d t  + f(x)
\end{align*}

 There is a rich literature on connection of divergences and distributions.
%  and on beta divergences. 
  In one study Banerjee et al. showed the
bijection between regular exponential family distributions and the Bregman divergences \cite{banerjee05}. As a special case of this bijection, connection of beta divergences and Tweedie distributions  has  been briefly remarked by \cite{cichocki09}  and has been specifically studied in a recent report \cite{Yilmaz2012}.

%Cichocki et al. published several seminal papers on generalized alpha and beta divergences, as such \cite{cichocki11,cichocki2010}. 

% \Note{
% Learning algm.
% Basu, egucihi,: \cite{basu1998},\cite{eguchi2001}, 
% 
%  cedric: \cite{cedric} 
%  
%  bizim icml c \cite{simsekli2013}, 
%  }
%  
 
 The attractive point of studying beta divergence is its generalization capability for learning algorithms. This capability has already been exploited in matrix and tensor factorizations by the researchers such as 
 \cite{fevotte10}, \cite{cichocki2010}, \cite{yilmaz11}, and very recently \cite{simsekli2013}. The key point of this generalization is minimization of beta divergence, whose underlying distribution is parametrized by a scalar index, as in the case of Tweedie models.  The net effect of our work would be replacing this distribution index by so-called {\it variance function} to deal with broad class of distributions.

To sum up, intended purpose of this paper is to facilitate generalized approach for designing learning algorithms that  optimize beta divergences. For this, the paper aims to gain some insight into various properties of beta divergence. The main contributions are as follows.
 
%Intended purpose of this paper is  to gain some insight into various properties of beta divergence. The main contributions are as follows.
\begin{itemize}
  \item The beta divergence is extended and linked to exponential dispersion models beyond Tweedie family. As a result, a statement like 'beta divergence for binomial distribution' becomes reasonable and it is equal to 
\begin{align*}
  d_\beta(x,\mu) &= x\log \frac{x}{\mu}  + (1-x)\log\frac{1-x}{1-\mu} 
\end{align*} 
   
  \item Various functions including beta divergences, alpha divergences, cumulant functions, dual cumulant functions are all expressed in similar definite integral forms.
 
  \item Derivations of certain properties of beta divergences  are simplified by using their integral representations. For example, connection of \Beta divergence $d_\beta(x,\mu)$ and its scaled form $d_\beta(x/c,\mu/c)$, that is 
\begin{align*}
 d_\beta(x,\mu) =  \frac{c^2}{f(c)} d_\beta(x/c,\mu/c) %, \qquad c\in \mathbb{R}_+
\end{align*}  
 can be simply shown by change of variables in the integral. 

\item The relation of beta divergence and unit deviance has already been studied in the scope of Tweedie models \cite{Yilmaz2012}. Here this connection is shown in a broader scope as
\begin{align}
  d_\nu(x,\mu) = 2 d_\beta(x,\mu) = 2  \int_\mu^x \frac{x-t}{ v(t)} dt
\end{align}
  
 \item Finally, we  present many examples that apply the results for the Tweedie models, which most could be considered as corollaries.
  
 \end{itemize}

This paper is organized as follows. 
Section 2 introduces the notation very briefly. Section 3 gives some background information about dispersion models, cumulant functions, Legendre duality and Bregaman divergences. Section 4 identifies basic elements such as canonical parameter $\theta$ and dual cumulant function in definite integral forms. Section 5 is all about beta divergence as its generic integral forms, along with its properties such as scaling and transformation. Section 6, then, links beta divergence to log-likelihood and statistical deviance. Finally, in Appendix, similar integral forms are given for the alpha divergence.

\section{Notation}

In particular, $d_\phi(x,\mu)$ is the Bregman divergence generated by the convex function $\phi(\cdot)$. Likewise, $d_f(x,\mu)$ denotes $f$-divergence generated by the convex function $f(\cdot)$. As special cases, $d_\alpha (x,\mu)$ and $d_\beta(x,\mu)$  denote alpha  and beta  divergences. Similarly $d_\nu(x,\mu)$ denotes the statistical deviance.  Related with the deviance, quasi-log-likelihood of the parametric model is denoted by $\LL{x}{\mu}$ whereas $\LL{x}{x}$ denotes the quasi-log-likelihood of the 'full' model. 
%In entropy section $\HE{x}{\mu}$ is the entropy of the variable $x$ whereas $\E{x}$ is its expectation. 
Without loss of generality, we  consider only scalar valued functions, and consider only univariate variables  whereas the work can easily be extended to multivariate case. %Also here, $\phi'$ and $\phi''$ denotes the first and the second derivatives of the function $\phi$.
In other words, we assume  $\mu,\theta \in \mathbb{R}$ and correspondingly when referred we consider only ordinary scalar product as $\mu\theta$  rather than the inner product $\left\langle \bm{\mu}, \bm{\theta }\right\rangle $.

\section{Background}

\subsection{Dispersion Models}

{\it (Reproductive) exponential dispersion models } $EDM(\theta,\varphi)$ are two-parameter linear exponential family distributions defined as
% and their history goes back to Tweedie's unnoticed work in 1947 \cite{tweedie1947}. Nelder and Wedderburn, in 1972, published GLMs \cite{nelder1972} without any references to Tweedie's work where their error distribution formulation was similar to Tweedie's formulation. In 1982 Morris used the term {\it Natural Exponential Models} (NEF) \cite{morris1982}, and 1987 Jorgensen \cite{jorgensen87} gave the name Tweedie distribution. By definition, {\it Exponential Dispersion Models } (EDM) are defined as 
\begin{align} 
   p(x; \theta,\varphi) = h(x,\varphi)exp \left\{ \varphi^{-1} \left( \theta x - \psi(\theta) \right)   \right\} 
\end{align}
Here $\theta \in \Theta$  is the {\it canonical (natural) parameter} with $\Theta$ being the canonical parameter domain. $\varphi$ is the {\it dispersion parameter} as  $\varphi > 0$ and $\psi$ is the {\it cumulant function} or {\it cumulant generator} that is inherently related to cumulant generating function. The {\it expectation (mean) parameter} is denoted by $\mu \in \Omega$ with $\Omega$ being the mean parameter domain. In this paper we assume mean parameter domain is identical to the convex support of the variable $X$ and thus we write $x,\mu \in \Omega$. The expectation parameter $\mu$ is the first cumulant and is tied  to the canonical parameter $\theta$ with the differential equation so-called {\it mean value mapping} \cite{jorgensen1997}
\begin{align}
   \mu(\theta) =  \frac{\partiald \psi(\theta)}{\partiald\theta \label{eq.expectation}}   
\end{align} 
which can be obtained after differentiating both side of the equation  wrt $\theta$
\begin{align}
  1=  \int  h(x,\varphi) \exp \{\varphi^{-1} \left( \theta x - \psi(\theta) \right)\} \;dx
\end{align}
Likewise the inverse function $(\mu(\theta))^{-1}$ is so-called {\it inverse mean value mapping} \cite{jorgensen1997}.

\begin{comment}

\Note{

Buradan cumulant nasıl çıkar:

\begin{align}
  1=  \int dx\; h(x,\varphi) \exp \{x\theta - \psi(\theta)\}
\end{align}

\begin{align}
  \log 1 &=  \log \int dx\; h(x,\varphi) \exp \{x\theta - \psi(\theta)\} \\
  0 &= - \psi(\theta) + \log \int dx\; h(x,\varphi) \exp \{x\theta \} \\
\end{align}
Then 
\begin{align}
   \psi(\theta) = \log \int dx\; h(x,\varphi) \exp (x\theta ) 
\end{align}

}
\end{comment}

\subsection{Dual Cumulant Function}

Similar to  one-to-one correspondence between $\theta$ and $\mu$  that span dual spaces $\Theta$ and $\Omega$, cumulant function $\psi(\cdot)$ has a conjugate dual form  denoted by $\phi(\cdot)$. This is the function that casts and specializes Bregman divergence to \Beta divergence. Similarly, when applied to Csizar $f$-divergence, it is specialized to Amari \Alpha divergence \cite{cichocki11}.  One other interesting propery of this function is that it is connected to the entropy of the underlying distribution \cite{wainwright08}, hence it is also called as (negative) entropy function. In this paper we simply refer as dual cumulant function, defined as follows

 \begin{align}\label{phi1}
    \phi(\mu) = \sup_{\theta \in \Theta} \{ \mu\theta - \psi(\theta)\}  
 \end{align}
 
%On the other hand, throughout the paper we will use simpler expression for the dual cumulant function that does not include the supremum given as follows.

%. The following lemma gives an alternative expression for dual cumulant function.  

Dual (conjugates) cumulant function has the following properties.
\begin{enumerate}
  \item Derivative of the dual cumulant function is the differential equation
\begin{align}\label{eqPhi} 
    \frac{\partiald \phi(\mu)}{\partiald \mu} = \theta(\mu)  =  \arg\sup_{\theta}  \mu\theta - \psi(\theta) 
\end{align}

\item The maximizing argument is as
\begin{align}\label{eqTheta1}
   \theta^* = \theta(\mu) = (\psi'(\theta))^{-1}(\mu)
\end{align}

\item The dual cumulant function can be computed as 
\begin{align}\label{dualCum}
  \phi(\mu) &=  \mu\theta(\mu) - \psi(\theta(\mu)) 
 %  &= \mu (\psi'(\theta))^{-1}(\mu) - \psi((\psi'(\theta))^{-1}(\mu))
\end{align}
 
\end{enumerate}

\begin{comment}

\begin{lem}Dual cumulant function can be expressed as 
\begin{align}\label{dualCum}
  \phi(\mu) &=  \mu\theta(\mu) - \psi(\theta(\mu)) 
 %  &= \mu (\psi'(\theta))^{-1}(\mu) - \psi((\psi'(\theta))^{-1}(\mu))
\end{align}
\end{lem}

\begin{proof}

Derivative of the dual cumulant function is the differential equation
\begin{align}\label{eqPhi} 
    \frac{\partiald \phi(\mu)}{\partiald \mu} = \theta(\mu)  =  \arg\max_{\theta}  \mu\theta - \psi(\theta) 
\end{align} 
due to, by definition, extending the supremum  and finding the maximizing argument. Then after optimizing by taking the derivative wrt $\theta$ and equating to zero 
\begin{align}
  \frac{d \theta(\mu)}{d \theta} = \mu-\psi'(\theta) = 0
\end{align}
and solving for $\theta$  
\begin{align}\label{eqTheta1}
   %\mu-\psi'(\theta) = 0 \qquad \Rightarrow \qquad 
   \theta^* = \theta(\mu) = (\psi'(\theta))^{-1}(\mu)
\end{align}
we obtain {\it mean-value mapping} given in eq \eqref{eq.expectation} as
$\mu(\theta) = d\psi(\theta)/ d\theta$. 
What this equation implies is that the maximizing $\theta$ is the one parametrized by $\mu$. 
Consequently we may replace $\theta$ by $\theta(\mu)$ 
in \eqref{phi1} and 
%Finally, after  maximizing $\theta^*=\theta(\mu)$ 
%we may equally strip off the supremum in \eqref{phi1} and 
write the dual cumulant as
\begin{align}\label{dualCum}
  \phi(\mu) &=  \mu\theta(\mu) - \psi(\theta(\mu)) 
 %  &= \mu (\psi'(\theta))^{-1}(\mu) - \psi((\psi'(\theta))^{-1}(\mu))
\end{align}

\end{proof}

\end{comment}

This transformation is known as Legendre transform. For an extensive mathematical treatment,  Legendre transform and convex analysis are in \cite{Rockafellar1970}. Dual space and duality of exponential families are in \cite{barndorff1978}  and in \cite{amari2000}.

 \begin{exam}
  Dual of the cumulant function $\psi(\theta)=\exp \theta$ for the Poisson can be found as below after applying \eqref{eqTheta1}
\begin{align}
   \theta(\mu) % &= 
   %(\psi'(\theta))^{-1}(\mu)  \\
    &= (\exp'(\theta))^{-1}(\mu) = \log \mu  
\end{align}
Then by \eqref{dualCum}, the dual cumulant function becomes
\begin{align}
   \phi(\mu) &=  \mu (\log \mu) - \exp(\log \mu) =  \mu \log \mu -  \mu
\end{align}
\begin{comment}

 which is equal to the derivative of the dual function
\begin{align}
 \phi'(\mu) = \log \mu  
\end{align}
Integrating back, we obtain dual cumulant function
\begin{align}
 \phi(\mu) = \int \log \mu \;d\mu = \mu\log \mu - \mu  + const
\end{align}
\end{comment}
  
\end{exam}

% \url{http://en.wikipedia.org/wiki/Convex_conjugate} 
% \url{https://www.ifor.math.ethz.ch/teaching/lectures/convexopt_ss10/Lecture_5}s

\subsection{Variance Functions}

The relationship between $\theta$  and  $\mu$ is more direct and  given as \cite{jorgensen87}
\begin{align}\label{eq.theta}
 \frac{ \partiald  \theta(\mu) }{ \partiald \mu } =
  \frac{ \partiald^2 \phi(\mu) }{ \partiald \mu^2 } =  
  v(\mu)^{-1}
\end{align}
 Here $v(\mu)$ is the {\it variance function} \cite{tweedie84,barlev86,jorgensen87}, and is related to the variance of the distribution by  dispersion parameter $\varphi$  as
 %$\varphi^{-1}$ as 
 \begin{align}
  Var(x) = \varphi v(\mu) 
\end{align}

As a special case of dispersion models, {\it Tweedie distributions} also called  Tweedie models assume that the variance function is in the form of power function  \cite{barlev86, jorgensen87} given as
\begin{align}
   v(\mu) = \mu^{p}
\end{align}
that fully characterizes one-parameter dispersion model. Here, the exponent is $p=0,1,2,3$ for well known distributions as Gaussian, Poisson, gamma and inverse Gaussian.
For $1<p<2$, they can be represented as the Poisson sum of gamma distributions so-called {\it compound Poisson distribution}. Indeed, they exist for all real values of $p$ except for the range $0 < p < 1$ \cite{jorgensen87}.

Variance functions play important roles in characterizing one-parameter distributions similar to the role  moment generating functions play. Given variance function,  density for dispersion model is completely identified by first identifying cumulant generating function and characteristic equation  and next inverting characteristic equation via Fourier inversion formula \cite{jorgensen1997}.  

\begin{comment}

\Note{
Yukarıda inverting KISMI???

For Tweedie models, for example, starting from the variance function cumulant function can be computed as, first, we find $\theta(\mu)$ by solving the differential equation
\begin{align}
 \frac{ \partiald  \theta(\mu) }{ \partiald \mu } = \frac{1}{v(\mu)}=  
  \mu^{-p}
\end{align}
and then invert it and find $\mu(\theta)$
\begin{align}
   \{\theta(\mu)\}^{-1} = \mu(\theta) =  \frac{\partiald \psi(\theta)}{\partiald\theta }   
\end{align}
and third we find the cumulant by integrating $\mu(\theta)$. Here arbitrary constants due to integration can be ignored safely \cite{barlev86,Yilmaz2012}.
}  

\Note{a) Bu yetmez. Belki full örnek. b( Bunu example altına mı almak daha iyi.)}
\Note{ p=1 ve p=2 için limitlerin hadamard ile alındığı söylemeliyiz.
Ancak bu limitler nerede gerekiyor?}
\Note{ Metedology: $\phi$, alpha beta, cumulant is generalized and tied to VF in integral form}

\end{comment}
  
\begin{table}[!t]
\caption{Distributions indexed by variance functions. 
%For Tweedie, $ p \not\in (0,1)$
} 
\label{tablelabel}  
\centering
\begin{tabular}{l|l||l|l}\hline
 
{\bf $v(\mu)$} & {\bf Tweedie} &   {\bf $v(\mu)$} &  {\bf Non-Tweedie} \\ \hline
$\mu^0$ &  Gauissan &  $\mu - \mu^2$ & Bernoulli \\ \hline 
$\mu^1$ & Poisson &  $\mu + \mu^2$ &   Neg. Binomial \\ \hline
$\mu^2$ & Gamma &  $1 + \mu^2$ &   Hyperbolic secant \\ \hline
$\mu^3$ & Inv. Gaussian &    & \\ \hline
$\mu^p$ & General Tweedie & &  \\ \hline

\end{tabular}
\end{table}

\subsection{Bregman Divergences}

%This section reviews briefly two general families of divergences that are {\it Bregman divergences} and {\it Csiszár $f$-divergences}.  

%\subsection{Bregman divergences} 

By definition, for any real valued differentiable convex function $\phi$ the Bregman divergence  \cite{bregman1967}
\begin{align}\label{eqBreg1}
   d_\phi(x,\mu) = \phi(x) - \phi(\mu) -(x-\mu) \phi'(\mu)
\end{align}  
It is equal to tail of first-order Taylor expansion of $\phi(x)$ at $\mu$. In addition, it enjoys convex duality and can be equally expressed as 
\begin{align}
  d_\phi(x,\mu) 
 % &= \phi(x) + \psi(\theta) - x \theta \\
  &= \phi(x) + \psi(\theta(\mu)) - x\theta(\mu)  
\end{align}
that can be showed by plugging $\phi(\mu)$ using the dual relation
 \begin{align} 
    \phi(\mu) &= 
    %\sup_{\theta \in \Theta} \{ \mu\theta - \psi(\theta)\}    
      \mu\theta(\mu) - \psi(\theta(\mu)) 
 \end{align}
 in \eqref{eqBreg1} and identifying $\phi'(\mu) = \theta(\mu)$ as given in \eqref{eqPhi}.

%d:\P3\Matlab\IEEE   
\begin{figure}[!t] 
\centering 
\subfigure{\includegraphics[width=1\columnwidth, height=0.5\columnwidth]{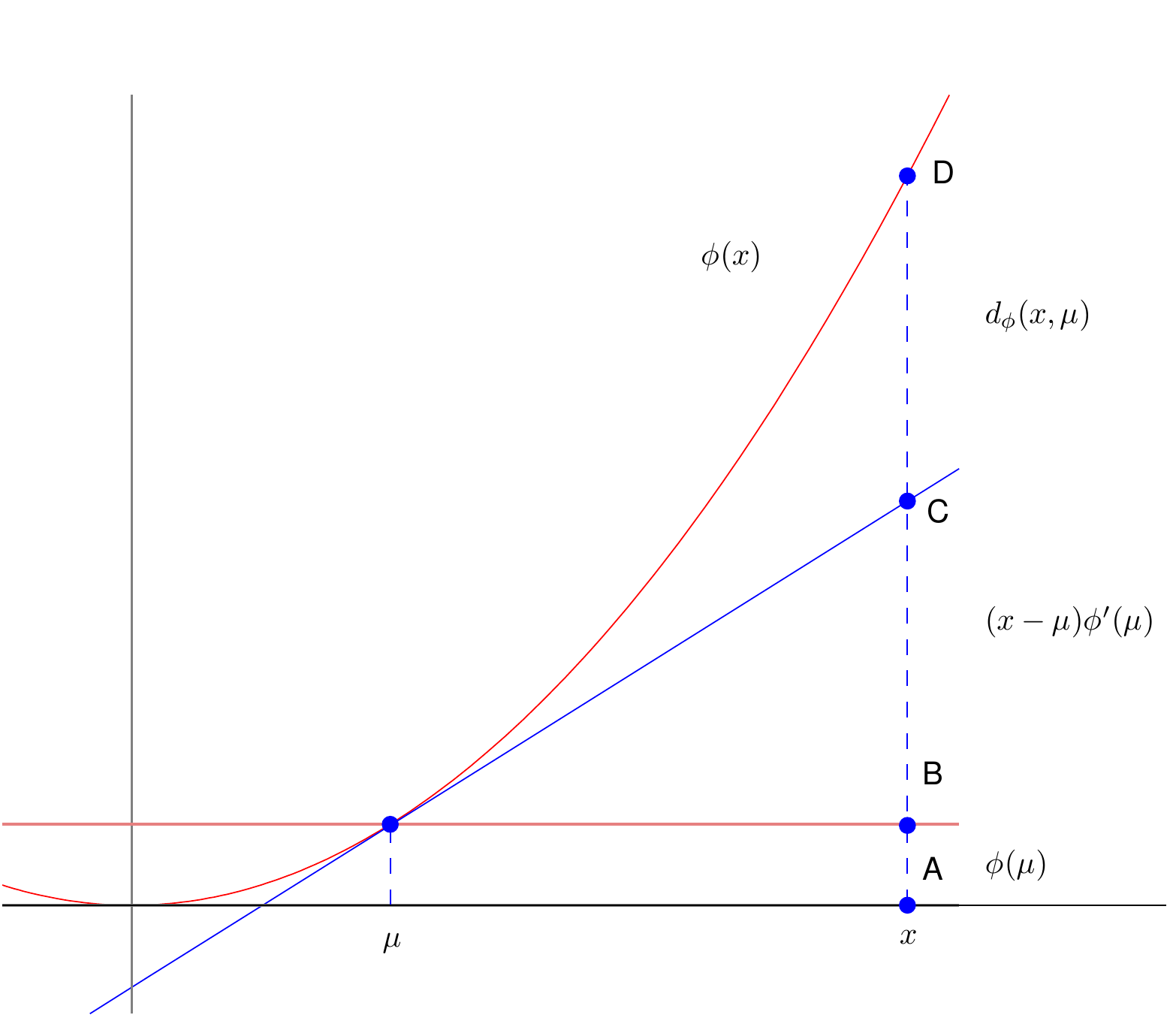}}
\caption{Beta divergence illustration $d_\beta(x,\mu)$ as the length of line segment $|CD|$. }
\label{fig:bregman}
\end{figure} 

The Bregman divergences are non-negative quantities as $d_\phi(x,\mu) \geq 0$ and equality holds only for $x=\mu$. However,  they provide neither symmetry nor triangular inequality in general and hence are not considered to be metrics. For a special choice of function $\phi$, Bregman divergence turns to  Euclidean distance that exhibits metric properties.

%\subsection{Csiszár's $f$-divergence }
\begin{comment}

The {\it $f$-divergences} are generalized KL divergences, and are  introduced independently by authors Csiszár \cite{csiszar1963}, Morimoto \cite{morimoto1963} and Ali \& Silvey \cite{ali1966} during 1960s. By definition, for any real valued convex function $f$  the $f$-divergence
%of $\mu$  from $x$
is defined as \cite{csiszar1963}
\begin{align}
  % d_f(x,\mu) = \sum \mu f(\frac{x}{\mu})
   d_f(x,\mu) =  \mu f(\frac{x}{\mu}) \qquad \text{with} \qquad f(1)  = 0
\end{align}
For the setting $x=1$, the divergence $d_f(1,\mu)$ becomes 
\begin{align}
  f^*(\mu) = \mu f(1/\mu)
\end{align}
where $f^*$ is called as Csiszár dual of the function $f$.

The Bregman and $f$-divergences are non-negative quantities as $d_\phi(x,\mu) \geq 0$ and $d_f(x,\mu) \geq 0$ whereas for $x=\mu$ they both become zero. However,  they provide neither symmetry nor triangular inequality in general and hence are not considered to be metrics. For special cases, Euclidean distance in Bregman family and Hellinger distance $f$-divergence family are metric and hence are qualified as distances.  

\end{comment}

\section{Basic Elements}

\subsection{The Canonical Parameter $\theta(\mu)$ }

We first derive the canonical parameter $\theta$ as the definite integral on interval $[\mu_0,\mu]$ by solving the differential equation
\begin{align}
  \frac{\partiald \theta(\mu)}{\partiald \mu} = \frac{1}{v(\mu)} \Rightarrow 
   \theta(\mu)= \int_{\mu_0}^{\mu} \frac{1}{v(t)} \partiald t 
\end{align} 
Here the choice of the lower bound $\mu_0$ is arbitrary but fixed. Indeed, any fixed value $\mu_0 \in \Omega$ can be specified as lower bound. As it will be clear, this choice has no effect on the computation of the beta divergence.

\subsection{Dual Cumulant Function $\phi(\mu)$}

%\Note{ilk defa bahsettiğimizde 
%To differentiate classical definition we qualify definition as 'generalized' whereas 
%in subsequent sections we simply drop it.}

Dual cumulant function plays central role when defining and generating beta divergences as  special cases of Bregman divergences. 

\begin{lem}
Let $v$ be  variance function of a dispersion model. Generalized dual cumulant function is equal to
\begin{align}
  \phi(\mu) &=  \int_{\mu_0}^\mu  \frac{\mu-t}{v(t)}    \partiald t
\end{align}
\end{lem}

\begin{proof}

By solving the differential equation, we obtain the  dual cumulant function $\phi(\mu)$ as  
\begin{align}
  \frac{\partiald \phi(\mu)}{\partiald \mu} = \theta(\mu)   \Rightarrow 
  \phi(\mu) = \int_{\mu_0}^\mu \theta(t)\partiald t 
\end{align} 
that is redefined after plugging $\theta(\mu)$ in 
\begin{align}
\phi(\mu) = \int_{\mu_0}^\mu  \left(\int_{\mu_0}^{t} \frac{1}{v(z)} \partiald z\right) \partiald t  
\end{align} 
Here bounds for the variables can be replaced as  
\begin{align}
  \mu_0 < z <t \qquad  &\Rightarrow \qquad   \mu_0 < z <\mu  \\ 
 \mu_0 < t <\mu \qquad &\Rightarrow \qquad  z < t <\mu  
\end{align}
where  we change the order of the integration and end up with the integral form for the dual cumulant  function $\phi$.

\begin{comment}

\Note{

By solving the differential equation, we obtain the  dual cumulant function $\phi(\mu)$ as  
\begin{align}
  \frac{\partiald \phi(\mu)}{\partiald \mu} = \theta(\mu)   \Rightarrow 
  \phi(\mu) = \int_{\mu_0}^\mu \theta(z)\partiald z 
\end{align} 
that is redefined after plugging $\theta(\mu)$ in 
\begin{align}
\phi(\mu) = \int_{\mu_0}^\mu  \left(\int_{\mu_0}^{z} \frac{1}{v(t)} \partiald t\right) \partiald z  
\end{align} 
Here bounds for the variables can be replaced as  
\begin{align}
  \mu_0 < t <z \qquad  &\Rightarrow \qquad   \mu_0 < t <\mu  \\ 
 \mu_0 < z <\mu \qquad &\Rightarrow \qquad  t < z <\mu  
\end{align}
where  we change the order of the integration and end up with the integral form for the entropy function.
}
\end{comment}

\end{proof}

As will be needed in further sections, the first and second derivatives of the dual cumulant function wrt $\mu$ are computed as follows
\begin{align}
  \phi'(\mu) &= \int_{\mu_0}^\mu  \frac{1}{v(t)} \;dt  \qquad
  \phi''(\mu) =  \frac{1}{v(\mu)} 
\end{align}
noting that the derivative of the dual cumulant function is inverse mean value mapping as $\phi'(\mu)= \theta(\mu)$.

The dual cumulant function $\phi$ corresponds to area under the curve $\theta(t)$ parametrized by $t \in \mathbb{R}$, in the interval $[t=\mu_0, \; t=\mu]$ given by the equation
\begin{align}
  \phi(\mu) = \int_{\mu_0}^\mu \theta(t)\partiald t 
\end{align} 
as illustrated by Figure \ref{fig:betaArea2}. The curve equation $\theta(t)$ in the figure is obtained as given in the following example.

\begin{exam}
For general case the curve equation is the cumulant function 
\begin{align}
  \theta(t) = \int_{t_0}^t \frac{1}{v(z)} \; dz
\end{align}
By setting variance function as $v(z)=z^p$ and the base as $t_0=1$, it is specialized for Tweedie models  as
\begin{align}
  \theta_p(t) &= \int_1^t  z^{-p} \;dz = \frac{t^{1-p} -1}{1-p}  
\end{align}
such that all the  curves generated by various values of $p$  meet at one common ground point $t=t_0=1$ where  
\begin{align}
  \theta_p(1) = 0
\end{align}
For $p=0,1,2$ the curve equations turn to
  \begin{align}  
  \theta(t)  = \left\{ \begin{array}{l}      
    \theta_0(t)=t-1  \\
    \theta_1(t)=\log t \\
    \theta_2(t)=- \frac{1}{t} + 1  \\ 
 \end{array}\right.
\end{align}

\end{exam} 

\begin{exam}\label{exDualCumulant1}
For Tweedie models  dual cumulant function $\phi(\mu)$ becomes as
\begin{align}\label{tweedie18}
   \phi(\mu) &= \int_{\mu_0}^{\mu} \frac{\mu-t}{t^p} \partiald t 
%  = \int_{\mu_0}^{\mu} \mu t^{-p}- t^{1-p} \partiald t \\
%  &= \frac{\mu t^{1-p}}{1-p} - \frac{t^{2-p}}{2-p}  \Big|_{\mu_0}^{\mu} \\
  %
   %&= \frac{ \mu^{2-p}}{(1-p)(2-p)}   - \frac{ \mu\mu_0^{1-p}}{1-p} + \frac{\mu_0^{2-p}}{2-p}   \\
   = \phi_1(\mu)   + \phi_0(\mu)
\end{align}
where the function $\phi_1(\mu)$ contains non-linear terms wrt $\mu$ as
\begin{align}
   \phi_1(\mu) &=  \frac{ \mu^{2-p}}{(1-p)(2-p)}  
\end{align}
whereas the function $\phi_0(\mu)$ contains linear terms wrt $\mu$
\begin{align}\label{linearMu}
  \phi_0(\mu) = -\mu \frac{ \mu_0^{1-p}}{1-p} + \frac{\mu_0^{2-p}}{2-p}
\end{align}

\begin{figure}[!t]   
\centering
\subfigure{\includegraphics[width=1\columnwidth]{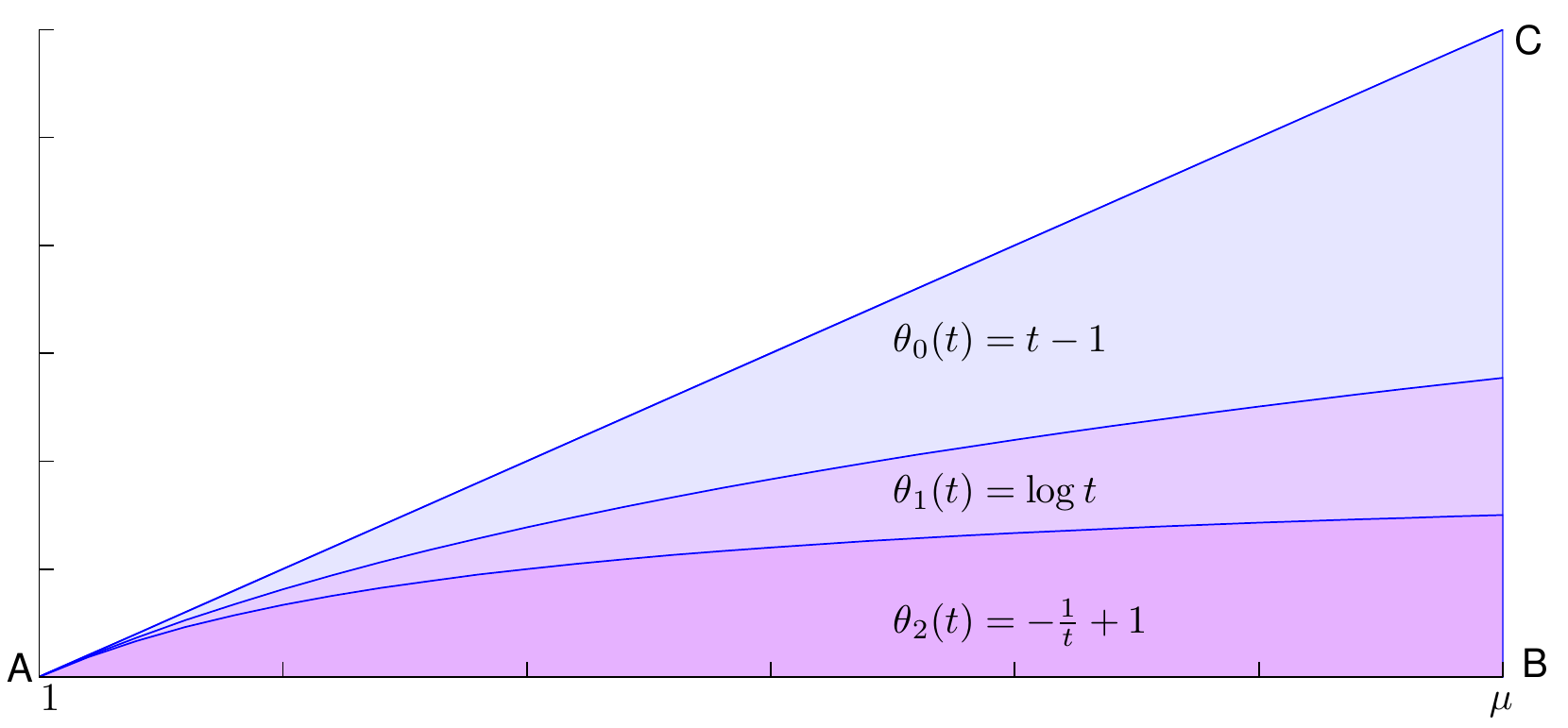}}
\caption{Figure illustrates dual cumulant function $\phi$ as the area under the curve $\theta_p(t)$. For example, for $p=0$ the function is associated with the Gaussian distribution and Euclidean distance, that corresponds to  the triangle ABC.  The area of the triangle is $1/2(\mu-1)^2$ since $|AB|=|BC|=\mu-1$ where this is equal to $\phi(\mu)=d_\beta(\mu,1)$. Here, $x$ axis is $t$ and $y$ axis is $\theta_p(t)$.
}
\label{fig:betaArea2}
\end{figure} 

The point of separating  linear and non-linear terms is that linear terms are canceled smoothly and disappear when generating beta divergence by the Bregman divergence as also reported by \cite{cichocki2010,reid2011}. Indeed, very often we only use non-linear part of dual cumulant function  $\phi_1(\mu)$ that has special values for $p=0,1,2$
 \begin{align*}
     \phi_1(\mu) &= \left\{ 
     \begin{array}{l l}
   %  \frac{x^{2-p}}{(1-p)(2-p)}    & \qquad  p \neq 1,2  \\
     \frac{1}{2}\mu^2   & \qquad  p=0  \\
    \mu\log \mu   & \qquad  p=1  \\ 
    -\log \mu  & \qquad  p=2  \\
    \end{array}
    \right. 
\end{align*}

Finally, as the general formulation, after setting $\mu_0=1$, dual cumulant function $\phi(\mu)$ for Tweedie models   becomes as
\begin{align*}
   \phi(\mu) &
   %\int_1^{\mu} \frac{\mu-t}{t^p} \partiald t 
  %= \int_1^{\mu} \mu t^{-p}- t^{1-p} \partiald t \\
  %&= \frac{\mu t^{1-p}}{1-p} - \frac{t^{2-p}}{2-p}  \Big|_1^{\mu} \\
  %&= \frac{ \mu^{2-p}}{1-p} - \frac{\mu^{2-p}}{2-p}  
  %- \frac{\mu }{1-p} + \frac{1}{2-p} \\  
  %= \frac{ \mu^{2-p}}{(1-p)(2-p)}  - \frac{\mu }{1-p} + \frac{1}{2-p}  
   = \frac{ \mu^{2-p}}{(1-p)(2-p)}  - \frac{\mu }{1-p} + \frac{1}{2-p}
\end{align*}
where the limits can be found by l'Hopitals. This form is  reported by various authors such as \cite{liese2006, cichocki11} with the index parameter  $q$ usually adjusted as $q=2-p$. 
\end{exam}

%\Note{ASLINA BUNU BELKI BAŞTA ($\phi$ nin ilk geçtiği yerde) SÖYLEMENK VE HER gerektiğinde $f()$ diye refer etmek iyi olabilir.

\begin{comment}

In the example above, the const term is
  \begin{align*}
   f(\mu_0) 
    &= - \int_{\mathcal{X}} dx\; p(x) x  \frac{\mu_0^{1-p}}{1-p} +
    \int_{\mathcal{X}} dx\; p(x)   \frac{\mu_0^{2-p}}{2-p} \\
    %
    &= - \mu  \frac{\mu_0^{1-p}}{1-p} +  \frac{\mu_0^{2-p}}{2-p} 
\end{align*}
\end{comment}

\begin{comment}
\begin{rem}
In the example above, the const term is
  \begin{align*}
   \phi_0(\mu) 
    &= - \int_{\mathcal{X}} dx\; p(x) x  \frac{\mu_0^{1-p}}{1-p} +
    \int_{\mathcal{X}} dx\; p(x)   \frac{\mu_0^{2-p}}{2-p} \\
    %
    &= - \mu  \frac{\mu_0^{1-p}}{1-p} +  \frac{\mu_0^{2-p}}{2-p} 
\end{align*}
\end{rem}
\end{comment}

\begin{comment}

\Note{
4. $\phi''$ ile eigenvalue ilişkisi.

5. Bunu ayır. İkincisi için yorum yap.
\begin{align}
  \phi'(\mu) &= \int_{\mu_0}^\mu  \frac{1}{v(t)} dt  \qquad
  \phi''(\mu) =  \frac{1}{v(\mu)} 
\end{align}  

6. Herhangi bir yere: we say VF generates beta divergence. Since VF=1/phi'' it is equvalent to saye phi generates beta divergence\ldots
}

\end{comment}

\section{Generalized Beta Divergence}

This section is all about  beta divergence where we first obtain its integral form and then study its scaling and translation properties.

\subsection{Extending Beta Divergence}

In the literature beta divergence is formulated as (with usually adjusted index parameter as $\beta=2-p$ or as $\beta=1-p$ \cite{cichocki09}) 
\begin{align}\label{beta1}
  d_\beta(x,\mu) &= \frac{x^{2-p}}{(1-p)(2-p)}  - \frac{x{\mu^{1-p}}}{1-p}  + \frac{\mu^{2-p}}{2-p}
\end{align}
This definition is linked to power variance functions of  Tweedie models. In this section its definition is extended beyond Tweedie models. 
 
\begin{defn}\label{defBeta}
Let $x,\mu \in \Omega$. (Generalized) beta divergence $d_\beta(x,\mu)$ is specialized Bregman divergence generated by (generalized) dual cumulant function $\phi$, which, in turn is induced by variance function $v(\mu)$.
\end{defn}

This definition implies that beta divergence is generated or induced by some variance functions. In this way, via the variance functions, beta divergence is linked to the exponential dispersion models. To differentiate this definition from well-known power function related definition it is qualified as 'generalized' whereas for the rest of the paper, we drop the term 'generalized' and simply refer as beta divergence.  

\begin{lem}
Beta divergence defined in Definition \ref{defBeta} $d_\beta(x,\mu)$
is equal to 
\begin{align}
  d_\beta(x,\mu) &=  \int_\mu^x  \frac{x-t}{v(t)}    \partiald t
\end{align} 
\end{lem}

\begin{proof}
By simply substituting the dual cumulant function and its first derivative
\begin{align}
  \phi(\mu) &=  \int_{\mu_0}^\mu  \frac{\mu-t}{v(t)}    \partiald t \qquad
  \phi'(\mu) =   \int_{\mu_0}^\mu   \frac{1}{v(t)}    \partiald t
\end{align}
in the Bregman divergence, we obtain the beta divergence.
\end{proof}

\begin{comment}
We note that generalized \Beta divergence is specialized to the generalized dual cumulant (entropy) function as a special case when $\mu=\mu_0$
\begin{align}
 \phi(x) = d_\beta(x|\mu_0) = \int_{\mu_0}^x \frac{x-t}{v(t)} \partiald t 
\end{align}
where we prefer the notation $d_\beta(x|\mu_0)$ since $\mu_0$ is no longer an independent variable but instead just a parameter and hence the notation   $d_\beta(x,\mu_0)$ becomes awkward. 
\end{comment} 
 
 We note that \Beta divergence is specialized to the  dual cumulant function as a special case when we measure divergence of $\mu$ from the base measure  $\mu_0$
\begin{align}
 \phi(\mu) = d_\beta(\mu,\mu_0) = \int_{\mu_0}^\mu \frac{\mu-t}{v(t)} \partiald t 
\end{align}
Here, we abuse  the notation and still use $d_\beta(\mu, \mu_0)$. This is indeed awkward since $\mu_0$ is no longer an independent variable but instead just a parameter and hence the notation   $d_\beta(\mu|\mu_0)$ would be better choice.

\begin{rem}
Remark that the integral form of the \Beta divergence can be obtained from Taylor expansion. Bregman divergence is equal to the first order Taylor expansion \cite{liese2006}
\begin{align}
  \phi(x) = \phi(\mu) + \phi'(\mu)(x-\mu) + R_\phi(x,\mu)
\end{align} 
where $R_\phi$ is the  remainder term expressed as
\begin{align}
  R_\phi(x,\mu) = \int_\mu^x (x-t)\phi''(t) \partiald t
\end{align}
The remainder is interpreted as the divergence from $x$ to $\mu$. As a special case for the dual cumulant function $\phi$ the Bregman divergence is specialized as \Beta divergence as
\begin{align}
   R_\phi(x,\mu) = d_\beta(x,\mu) = \int_\mu^x \frac{x-t}{v(t)} \partiald t
\end{align}
where $\phi''(t)$ is the inverse variance function $v(\mu)^{-1}$. 
\end{rem}

\begin{exam}

By following  Definition \ref{defBeta} we  easily find beta divergences for distributions including Tweedie and non-Tweedie families by decomposing the integral into two parts and computing each separately as
\begin{align*}
  d_\beta(x,\mu) &=  \int_\mu^x  \frac{x-t}{v(t)}    \partiald t 
  =  x \int_\mu^x  \frac{1}{v(t)}    \partiald t -
   \int_\mu^x  \frac{t}{v(t)}    \partiald t
\end{align*}
For example, for Tweedie models with VF $v(\mu)=\mu^p$ this results to the classical form  of beta divergence as in \eqref{beta1}. For others, for Bernoulli (Binomial with $m=1$) with VF given as 
  $v(\mu) = \mu - \mu^2$ \cite{mcculloch89} generates the \Beta divergence   
\begin{align}
  d_\beta(x,\mu) &= x\log \frac{x}{\mu}  + (1-x)\log\frac{1-x}{1-\mu} 
\end{align}
whereas we compute \Beta divergences for  negative binomial distribution with VF $v(\mu) = \mu + \mu^2$ 
\begin{align}
  d_\beta(x,\mu)&=  x\log \frac{x(1+\mu)}{\mu(1+x)}  + \log\frac{1+\mu}{1+x} 
\end{align}
and for hyperbolic secant distribution with VF $v(\mu) = 1 + \mu^2$  
\begin{align*}
  d_\beta(x,\mu) = x(\arctan  x - \arctan  \mu) + \frac{1}{2}\log\frac{1+\mu^2}{1+x^2}
\end{align*}
by simply  following integrals (ignoring the constant)
\begin{align*}
  \int \frac{dt}{1+t^2} = \arctan t  \qquad  \int \frac{t}{1+t^2}dt = \frac{1}{2}\log(1+t^2) 
\end{align*}
\end{exam}

\subsection{Geometric Interpretation of Beta Divergence}

As a special case of Bregman divergence, beta divergence can be interpreted as the length of the line segment in Figure \ref{fig:bregman} as already known. As an alternative interpretation, beta divergence $d_\beta(x,\mu)$ for $x \geq \mu$ is  area of the region surrounded by at the top the curve $\theta(t)$, at the bottom by the horizontal line $y=\theta(\mu)$ and at right by the vertical line $x=\theta(x)$. To show this, we use the identities 
\begin{align}
  \phi(\mu) = \int_{\mu_0}^{\mu} \theta(t) \; dt \qquad \text{and} \qquad \phi'(\mu) =  \theta(\mu) 
\end{align}
in the Bregman divergence as 
\begin{align*}
  d_\beta(x,\mu) &= \phi(x) - \phi(\mu) - (x-\mu)\phi'(\mu) \\
  &= \int_{\mu_0}^{x} \theta(t) \; dt - \int_{\mu_0}^{\mu} \theta(t) \; dt - (x-\mu)\theta(\mu) \\
  &= \left(\int_{\mu}^{x} \theta(t) \; dt \right)  - (x-\mu)\theta(\mu) 
%  &= \text{Area of BCFD} - \text{Area of BCED} = \text{Area of DEF}
\end{align*}
where this subtraction corresponds to the area in Figure \ref{fig:betaArea1} as
\begin{align*} 
  d_\beta(x,\mu) 
  &= \text{Area of BCFD} - \text{Area of BCED} = \text{Area of DEF}
\end{align*}

In additon, in Figure \ref{fig:betaArea1} the dual cumulant functions are identified as
\begin{align} 
  \phi(\mu) &= \text{Area of ABD} \\
  \phi(x) &= \text{Area of ACF} 
\end{align}

\begin{figure}[!t]
\centering
\subfigure{\includegraphics[width=1\columnwidth]{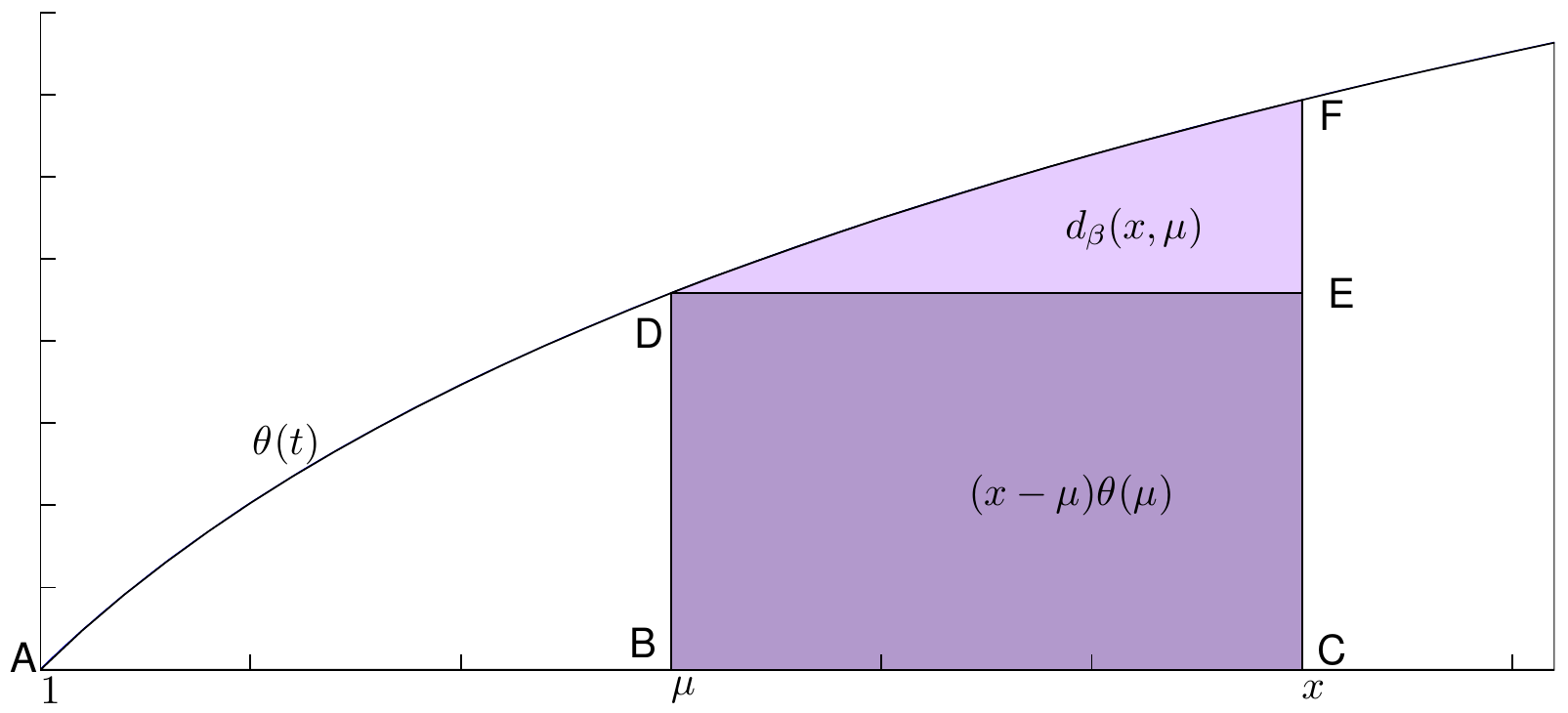}}
\caption{Figure illustrates beta divergence $d_\beta(x,\mu)$ as area of region DEF. Area of ABD is $\phi(\mu)$ and area of ACF is function $\phi(x)$ whereas area of the rectangle BCED is $(x-\mu)\theta(\mu)$.  Here, $x$ axis is $t$ and $y$ axis is $\theta_p(t)$.}
\label{fig:betaArea1}
\end{figure}

\subsection{Scaling of Beta Divergence}

In this section we find the scaling of beta divergence by $1/c$, i.e. we relate  
\begin{align}
% d_\beta(x,\mu) &\propto  d_\beta(x/c,\mu/c) \\
   \int_\mu^x \frac{x-t}{v(t)} dt \quad &\propto  \quad  \int_{\mu/c}^{x/c} \frac{x/c-t}{v(t)} dt 
\end{align} 
One immediate use of the result of scaling analysis is that by substituting $c=\mu$, we connect beta divergence to alpha divergence as shown later in this paper.

\begin{lem}
Let $x,\mu \in \Omega$, and $v(\mu)$ be the variance function. Then  \Beta divergence $d_\beta(x,\mu)$ and its scaled form $d_\beta(x/c,\mu/c)$ with the scalar $c\in \mathbb{R}_+$ are related as
\begin{align}
 d_\beta(x,\mu) =  \frac{c^2}{f(c)} d_\beta(x/c,\mu/c)
\end{align}  
where the variance function $v(\cdot)$ is written as in the following form
 \begin{align}
    f(c) =  \frac{v(ct)}{v( t)  }. 
\end{align}
\end{lem}

\begin{proof}

Here we use change of variable of the integrals. Consider the integral
\begin{align}\label{eq.scaleBeta1}
  \int_\mu^x f(  g(t) ) g'(t) dt
\end{align}
where the function $g$ is set to 
\begin{align}
  g(t) = \frac{t}{c} \qquad with \qquad g'(t) = \frac{1}{c} 
\end{align}
and we look for $f$. The hint is that we want $f(  g(t) ) g'(t)$ to match to  beta divergence function inside the integral 
\begin{align}
  f(\frac{t}{c})\frac{1}{c} = \frac{x-t}{v(t)} \qquad or \qquad  f(r) = c \frac{x-c r}{v(c r)}
\end{align}
where $r=t/c$. Thus, beta divergence turns to 
\begin{align*}
  d_\beta(x,\mu) &= \int_\mu^x f(  g(t) ) g'(t) dt = \int_{g(\mu)}^{g(x)} f(t ) dt \\
  &=  \int_{\mu/c}^{x/c} c \frac{x-c t}{v(c t)} dt  =
   c^2 \int_{\mu/c}^{x/c}  \frac{x/c-t}{v(c t)} dt
\end{align*}
 Next, consider a special decomposition of the variance function
 \begin{align}
   v(c t) = f(c)v(t) 
\end{align} 
Using this beta divergence can be expressed as 
\begin{align*}
  d_\beta(x,\mu) &=  
   c^2 \int_{\mu/c}^{x/c}  \frac{x/c-t}{v(c t)} dt = 
    \frac{c^2}{f(c)} \int_{\mu/c}^{x/c}  \frac{x/c-t}{v(t)} dt  
\end{align*}
that turns to the relation
\begin{align}
  d_\beta(x,\mu) =  \frac{c^2}{f(c)} d_\beta(x/c,\mu/c)
\end{align}

\end{proof}

A special case is when $c=\mu$ 
\begin{align}
 d_\beta(x,\mu) =  \frac{\mu^2}{f(\mu)} d_\beta(x/\mu,1)
\end{align}
that relates beta divergence to alpha divergence as will be shown. 

\begin{exam}
For Tweedie models where $v(\mu)=\mu^p$, $f(c)=c^p$ since
\begin{align}
  v(c \mu)=(c \mu)^p =  c^p v(\mu) 
\end{align} 
and hence
\begin{align}
  d_\beta(x,c) =  \frac{c^2}{c^p} d_\beta(x/c,\mu/c) = 
  c^{2-p} d_\beta(x/c,\mu/c)  
\end{align} 
For specifically $c=\mu$, we have
\begin{align}
  d_\beta(x,\mu) =   \mu^{2-p} d_\beta(x/\mu,1)  
\end{align} 
\end{exam}

\subsection{Translation of Beta Divergence}

Similar to scaling property here we analyse translation property of the beta divergence by the scalar $c \in \mathbb{R}$. In other words, now, we relate   
\begin{align}
% d_\beta(x,\mu) &\propto  d_\beta(x+c,\mu+c) \\
   \int_\mu^x \frac{x-t}{v(t)} dt \quad &\propto \quad    \int_{\mu+c}^{x+c} \frac{x+c-t}{v(t)} dt 
\end{align}

\begin{lem}
%Let $x,\mu \in \Omega$, $v(\mu)$ be the variance function, and $d_\beta(x,\mu)$ be the beta divergence. 
\Beta divergence $d_\beta(x,\mu)$ and its translated form $d_\beta(x+c,\mu+c)$ with the scalar $c\in \mathbb{R}$ are related as
\begin{align}
   d_\beta(x,\mu) =  \frac{1}{f(c)} d_\beta(x+c,\mu+c)
\end{align}
where the function $f$ is the ratio of 
%the inverse 
translated variance function to the original variance function as   
%with the variance function $v(\mu)$ and $c$ as a translation term that decomposes as a special case
 \begin{align}
   f(c) =  \frac{v(\mu-c)}{v( \mu)  }.
\end{align}
\end{lem}

\begin{exam}
For an immediate example, for the Gaussian distribution the variance function $v(\mu)=1$  implies that $f(c)=1$. Thus, for any displacement $c$ we have 
\begin{align}
   d_\beta(x,\mu) =   d_\beta(x+c,\mu+c)
\end{align}
which means beta divergence is invariant for translation under the Gaussian distribution (or equivalently saying under the Euclidean distance).
\end{exam}

\begin{proof}
The proof is fairly similar to that of scaling of \Beta divergence above. To apply change of variable of the integrals, we choose function $g$ as 
\begin{align}
  g(t) = t+c \qquad with \qquad g'(t) = 1 
\end{align}
and function $f$ becomes  
\begin{align}
  f(t+c) = \frac{x-t}{v(t)}
\end{align}
Then by plugging $r=t+c$ or $t=r-c$ and getting rid of $t$ it is simplified as 
\begin{align}
  f(r) =  \frac{x - r + c }{v(r-c)}
\end{align}  
Finally the integral turns to beta divergence as
\begin{align}
  d_\beta(x,\mu) &=  \int_{g(\mu)}^{g(x)} f(t ) dt \\
  &=  \int_{\mu+c}^{x+c}  \frac{(x+c) - t}{v(t-c)} dt
\end{align}
To simplify this integration further we consider a special case
 \begin{align}
   v(t-c) = f(c)v(t) 
\end{align} 
By this  beta divergence can be expressed as 
\begin{align}
  d_\beta(x,\mu) &=   
    \frac{1}{f(c)} \int_{\mu+c}^{x+c}  \frac{(x+c)-t}{v(t)} dt  
\end{align}
where we obtain the relation
\begin{align}
  d_\beta(x,\mu) =  \frac{1}{f(c)} d_\beta(x+c,\mu+c) 
\end{align}

\end{proof}

\begin{exam}
For exponential variance functions $v(\mu)=\gamma^\mu$, for any $\gamma \in \mathbb{R}_+$, we find $f(c)=\gamma^{-c}$ since 
\begin{align*}
  v(\mu-c)= \gamma^{\mu-c} = \gamma^{-c} \gamma^\mu   =  \gamma^{-c} v(\mu)   
\end{align*} 
 Hence translated beta divergence by $c$ becomes
\begin{align}
 d_\beta(x,\mu) % &=  \frac{1}{\gamma^{-c}} d_\beta(x+c,\mu+c) \\
   &=  \gamma^c d_\beta(x+c,\mu+c) 
\end{align}

\end{exam}

\section{Log-Likelihood, Deviance and Beta Divergence}

This section   links   beta divergence to log-likelihood and statistical deviance.

\subsection{Unit Quasi-Log-Likelihood}

Consider one-parameter dispersion model parameterized by $\mu$ 
%\begin{align} 
%   p(x| \theta,\varphi) = h(x,\varphi)exp \left\{ \varphi \left( \theta x - \psi(\theta)% \right)   \right\} 
%\end{align}
\begin{align*} 
   p(x; \mu) = h(x,\varphi)exp \left\{ \varphi^{-1} \left( \theta(\mu) x - \psi(\theta(\mu)) \right)   \right\} 
\end{align*}
that the dispersion $\varphi$ is assumed to be an arbitrary but fixed parameter. %{\it Unit quasi-log-likelihood} is defined as below 

\begin{defn}\label{defQuasiLL}
Let $x,\mu \in \Omega$, $\theta \in \Theta$ and $\Omega=\Theta$. Let $\mu$ be ML estimate of $x$, i.e. $\mu$ and $\theta$ are duals, then  unit quasi-log-likelihood denoted by $\LL{x}{\mu}$ is defined as
\begin{align}
  \LL{x}{\mu} =  \theta(\mu) x - \psi(\theta(\mu))
\end{align}   
\end{defn}
The unit quasi-log-likelihood is  quasi-log-likelihood given by  Wedderburn \cite{wedderburn1974} with the addition of 'unit' term.
Compared to a log-likelihood function, the constant term $\log h(x,\varphi)$ wrt $\mu$ is dropped and it is scaled by the dispersion $\varphi$ that turns the quasi-log-likelihood into 'unit' quasi-log-likelihood.  
On the other hand, $\LL{x}{\mu}$ has many properties in common with the  log-likelihood as given by Wedderburn \cite{wedderburn1974}.  
%In addition, similar forms of the log-likelihood definitions can be found in \cite{jorgensen1997}

\begin{lem}
Unit quasi-log-likelihood function $\LL{x}{\mu}$  is equal to 
\begin{align}
  \LL{x}{\mu} &= \int_{\mu_0}^{\mu} \frac{x-t}{v(t)} \partiald t
\end{align} 

\end{lem}

\begin{proof}

First, getting  rid of  $\psi(\theta(\mu))$ as plugging the duality
\begin{align}
  \psi(\theta(\mu)) = \mu\theta(\mu) - \phi(\mu) 
\end{align}
in the definition of $\LL{x}{\mu}$
\begin{align*}
  \LL{x}{\mu} &= \theta(\mu) x - \psi(\theta(\mu)) =
   \theta(\mu) x - \mu\theta(\mu) + \phi(\mu)  \\ 
   &= \theta(\mu) (x - \mu) + \phi(\mu) 
\end{align*}
then, second,  by substituting the terms $\theta(\mu)$ and $\phi(\mu)$, we obtain 
\begin{align}
 \LL{x}{\mu} &= (x-\mu) \int_{\mu_0}^\mu \frac{1}{v(t)}dt + \int_{\mu_0}^\mu \frac{\mu-t}{v(t)}dt \\ 
 &=  \int_{\mu_0}^\mu \frac{x-t}{v(t)}dt 
\end{align}

\end{proof}

For practical purposes in the examples, we take the lower bound $\mu_0=0$. 

We show that there are following connections between dual cumulant function and unit quasi-log-likelihood.
\begin{cor}
Let $\phi$ be the dual cumulant function and $\mathcal{L}$ be the unit quasi-log-likelihood. Then there are the following connections between $\phi(\cdot)$ and $\LL{\cdot}{\cdot}$ 
\begin{align}
  \text{i}) \qquad \phi(\mu) &= \LL{\mu}{\mu} = \int_{\mu_0}^\mu \frac{\mu-t}{v(t)} dt \\
  \text{ii}) \qquad \phi(x) &= \LL{x}{x} = \int_{\mu_0}^x \frac{x-t}{v(t)} dt 
\end{align}
\end{cor}

Here $\LL{x}{x}$ is the quasi-log-likelihood of the 'full' model where data speaks about data whereas $\LL{x}{\mu}$ is the quasi-log-likelihood of the parametric model where data  speaks about model parameters. $\LL{\mu}{\mu}$ is the same function as  $\LL{x}{x}$ where it only makes sense when both of  $x$ and $\mu$ are independent variables of the same function such as $d(x,\mu)$ that can be written as differences of    $\LL{x}{x}$ and $\LL{\mu}{\mu}$.
  
\begin{proof}
Proof of  i) and ii) immediately follow from the integral definitions of $\phi$ and $\LL{}{}$.
\end{proof}

Scaled quasi-log-likelihood  $\LL{x}{\mu}$  provides the following properties.

\begin{cor}
The unit quasi-log-likelihood $\LL{x}{\mu}$ and its first two derivatives have the following  expectations
\begin{align}
  &\text{i)} \quad  \E{\LL{x}{\mu}} =  \phi(\mu)  \\
  &\text{ii)} \quad \E{ \frac{\partial \LL{x}{\mu}}{ \partial \mu}  } = 0 \\
  &\text{iii)} \quad \E{ \frac{\partial^2 \LL{x}{\mu}}{ \partial \mu^2}  } = -\frac{1}{v(\mu)} 
\end{align}  
\end{cor}

\begin{proof}
For i),  simply we take the expectation of $\LL{x}{\mu}$  
\begin{align}
  %\phi(\mu) &= \mu \theta(\mu) - \psi(\theta(\mu)) \\
 \E{ \LL{x}{\mu}}  &= \E{x \theta(\mu) - \psi(\theta(\mu))} = \phi(\mu). 
\end{align}  
For ii), it  turns to the expectation
\begin{align}
  \E{ \frac{x-\mu}{v({\mu})} } = 0
\end{align}
For iii), it is the expectation of the second derivative as
 \begin{align}
  \E{-\frac{1}{v(\mu)} + (x-\mu)  (-1)\frac{1}{v(\mu)^2}v'(\mu) }
  = -\frac{1}{v(\mu)}
\end{align}
In fact, $\LL{x}{\mu}$ is scaled form of quasi-log-likelihood, and thus they provide  similar properties, and properties ii) and iii) have already been shown in \cite{wedderburn1974}.

\end{proof}

%Note that the property iii) turns to Fisher Information matrix in multivariate case.

\subsection{Statistical Deviance}

By  definition,  {\it unit deviance}  is two times of the log-likelihood ratio scaled by the dispersion, i.e. twice ratio of unit log-likelihood of the 'full' model to the that of parametric model \cite{mcculloch89,jorgensen1997}.

\begin{defn}
Let  $x,\mu \in \Omega$, $\mu$ be ML estimate of $x$ and $\LL{x}{\mu}$ be the unit quasi-log-likelihood function. Then unit deviance denoted by $d_\nu(x,\mu)$ is 
\begin{align}
  d_\nu(x,\mu) = 2 \left\{ \LL{x}{x} - \LL{x}{\mu} \right\}  
\end{align}
\end{defn}

This definition leads to the integral representation of the unit deviance that is also given in \cite{mcculloch89,jorgensen1997}.

\begin{lem}  
The unit deviance denoted by $d_\nu(x,\mu)$ is equal to 
\begin{align}
    d_\nu(x,\mu) = 2  \int_\mu^x \frac{x-t}{ v(t)} dt
\end{align} 

\end{lem}

\begin{proof}
 Proof is simply completed  by subtracting integral forms of $\LL{x}{x}$ and $\LL{x}{\mu}$.
\end{proof}

The following lemma  states the equivalence of beta divergence and unit deviance concepts.  
\begin{lem}\label{betaDivergence}
Let $d_\beta(x,\mu)$ be the beta divergence and $d_\nu(x,\mu)$ be the unit (scaled) deviance. Then unit deviance is twice of the beta divergence as 
\begin{align}\label{eq.devBeta}
  d_\nu(x,\mu) = 2 d_\beta(x,\mu) = 2  \int_\mu^x \frac{x-t}{ v(t)} dt
\end{align}
\begin{proof}
Proof is immediately implied by the equality of the integral representations of both functions.
\end{proof}

\end{lem}

This lemma implies the connection between beta divergence and unit quasi-log-likelihood.
 
\begin{cor}
Beta divergence $d_\beta(x,\mu)$ is equal to the following difference 
\begin{align}
  d_\beta(x,\mu) = \LL{x}{x} - \LL{x}{\mu} 
\end{align}
\end{cor}

\begin{proof}
Proof is trivial consequence of the lemma above.
\end{proof}
We note that this result is independent of the lower bound $\mu_0$.

\subsection{Density Representation via Beta Divergence}

One immediate consequence of Lemma \ref{betaDivergence} that states equivalence of beta divergence and divergence is that the standard form of density of dispersion model $DM(\mu, \varphi)$ \cite{jorgensen1997}
\begin{align}
    p(x; \mu,\varphi) = g(x, \varphi) \exp \left\{-\frac{1}{2}\varphi^{-1} d_\nu(x, \mu) \right\}
\end{align}
can be written in terms of beta divergence as
\begin{align}\label{betaDensity}
    p(x; \mu,\varphi) = g(x, \varphi) \exp \left\{-\varphi^{-1} d_\beta(x, \mu) \right\}
\end{align}

\begin{rem}
Note that by plugging dual form of beta divergence 
\begin{align}
    x \theta - \psi(\theta)   =  \phi(x)-d_\beta(x,\mu) 
\end{align}
in density of exponential dispersion models
\begin{align*}
  p(x; \theta,\varphi) &= h(x,\varphi) \exp \{\varphi^{-1} (x\theta - \psi(\theta)) \}  
\end{align*}
we obtain  density form  in \eqref{betaDensity}. This is special case of the generalized method in \cite{banerjee05} that exploits the bijection between  Bregman divergences and exponential family of distributions. Here the functions $h$ and $g$ are related as
\begin{align}
  g(x,\varphi) = h(x,\varphi)  \exp\{\varphi^{-1} \phi(x)\}
\end{align}

\end{rem}

In the followings we illustrate  various densities expressed via beta divergences \cite{Yilmaz2012}.

\begin{exam}
The density of the Gaussian distribution with dispersion  $\varphi=\sigma^2$ is given as  \cite{jorgensen1997}
\begin{align*}
   p(x;\mu,\sigma^2) &=  \underbrace{(2\pi\sigma^2)^{-\frac{1}{2}} \exp\frac{-x^2}{2\sigma^2}}_{h(x,\varphi)} 
     \exp \Big\{\frac{1}{\sigma^2} \Big( x \underbrace{\mu}_{\theta(\mu)} - \underbrace{\frac{\mu^2}{2}}_{\psi(\theta(\mu))}  \Big\} 
\end{align*}   
that is equivalently expressed as  via \Beta divergence  
\begin{align*}
   p(x;\mu,\sigma^2) &=  \underbrace{(2\pi\sigma^2)^{-\frac{1}{2}} }_{g(x,\varphi)} 
       \exp \Big\{-  \frac{1}{\sigma^2}
       d_\beta(x,\mu) \Big\}
\end{align*}

The density of the gamma distribution with $a$ and $b$ shape and (inverse) scale parameters is  
\begin{align}
  p(x; a,b) = \frac{x^{a-1}}{\Gamma(a)} \exp \{-b x + a \log b\}
\end{align}
Using the gamma distribution convention such that $\mu=a/b$ and $Var(x) = a/b^2$ dispersion becomes $\varphi=1/a$
%\begin{align}
%  Var(x) = a/b^2 = \varphi \mu^2 = \varphi (a/b)^2 \quad \Rightarrow \quad \varphi=1/a
%\end{align}
%that is inversely related as  
%$a=1/\varphi$ and $b=1/(\mu\varphi)$. 
we re-write the density in terms of mean and inverse dispersion as 
\begin{align}
  p(x; \mu,a) &= \underbrace{\frac{x^{a-1}}{\Gamma(a)} a^a}_{h(x,\varphi)} \exp \{a ( \underbrace{-\frac{1}{\mu}}_{\theta(\mu)}x  -  \underbrace{\log \mu)}_{\psi(\theta(\mu)) }\}   
\end{align}
Then, by adding and subtracting $\log x + 1$ in the exponent we obtain
\begin{align}
p(x;\mu,a) 
  &= \underbrace{\frac{x^{-1} a^{a} \exp(- a) }{\Gamma(a)}}_{g(x,\varphi)}    \exp \{-a  d_\beta(x,\mu) \}  
\end{align}

For the Poisson distribution with the dispersion  $\varphi=1$ the density is \cite{jorgensen1997}
\begin{align}
  p(x;\mu) = \underbrace{\frac{1}{x!}}_{h(x)} \exp \{x \underbrace{\log\mu}_{\theta(\mu)} - \underbrace{\mu}_{\psi(\theta(\mu))}\}
\end{align}
 that by adding and subtracting $x\log x -x$ in the exponent we obtain  beta representation of the density
\begin{align}
  p(x;\mu) =  \underbrace{\frac{x^x \exp x}{x!}}_{g(x)} \exp \{-d_\beta(x,\mu)\}
\end{align}
\end{exam}

\subsection{Expectation of Beta Divergence}

 An interesting quantity is the expected beta divergence. It opens connections to the relating beta divergence and Jensen divergence. 
 %It has also  connection to Bregman Information \cite{banerjee05}. 

\begin{lem}\label{lemExpBeta} 
Expectation of beta divergence is 
\begin{align}\label{expBeta}
  \E{d_\beta(x,\mu)} &= \E{\phi_1(x)} - \phi_1(\mu) 
\end{align}
where $\phi_1(\cdot)$ is the non linear part of the dual cumulant function as defined earlier. 
\end{lem}
In this context, $\phi_1(\cdot)$ can be simply regared as dual cumulant function.
\begin{proof}
 By taking the expectation of Bregman divergence
\begin{align}
  \E{d_\beta(x,\mu)} &= \E{\phi(x) - \phi(\mu) - (x-\mu)\phi'(\mu)} \\
  &= \E{\phi(x)} - \phi(\mu) - \E{(x-\mu)}\phi'(\mu) \\
  &= \E{\phi(x)} - \phi(\mu) 
\end{align}
Now by plugging $\phi = \phi_1 + \phi_0$ we end up with
\begin{align*}
  \E{d_\beta(x,\mu)} &= 
  %\E{\phi(x)} - \phi(\mu) \\
  % \E{\phi_1(x) + \phi_0(x)} - ( \phi_1(\mu) + \phi_0(\mu)) \\
   \E{\phi_1(x)} + \E{\phi_0(x)} -  \phi_1(\mu) - \phi_0(\mu) \\
  &= \E{\phi_1(x)} + \phi_0(\mu) -  \phi_1(\mu) - \phi_0(\mu) \\
  &= \E{\phi_1(x)}  -  \phi_1(\mu) 
\end{align*}

Note that  $\E{\phi_0(x)} = \E{\phi_0(\mu)}$ due to that $\E{x}=\mu$ and that $\phi_0(x)$ has only linear terms wrt $x$.
\end{proof}

\begin{comment}
\begin{lem}\label{lemExpBetaXXX} 
Expectation of beta divergence is 
\begin{align}\label{expBeta}
  \E{d_\beta(x,\mu)} &= \E{\phi(x)} - \phi(\mu) 
 % &= \E{ \LL{x}{x}} - \LL{\mu}{\mu} 
\end{align}
\begin{proof}
 By taking the expectation of Bregman divergence
\begin{align}
  \E{d_\beta(x,\mu)} &= \E{\phi(x) - \phi(\mu) - (x-\mu)\phi'(\mu)} \\
  &= \E{\phi(x)} - \phi(\mu) - \E{(x-\mu)}\phi'(\mu) \\
  &= \E{\phi(x)} - \phi(\mu) 
\end{align}
%Then we use the connections between $\phi$ and $\mathcal{L}$.
\end{proof}
\end{lem}
\end{comment}

\begin{cor}
Expected beta divergence is equal to Jensen gap for the dual cumulant function $\phi_1$ as
\begin{align}
   \E{d_\beta(x,\mu)} = \E{\phi_1(x)} - \phi_1( \E{x}) 
\end{align}

\end{cor}
\begin{proof}
 Jensen inequality for the convex function $\phi_1$ is
\begin{align}
   \E{\phi_1(x)} \geq \phi_1( \E{x}) = \phi_1(\mu) 
\end{align}
where the gap is the expected beta divergence given in \eqref{expBeta}. 

\end{proof}

\begin{exam}
Expectation of beta divergence for Tweedie models is computed by using Lemma \ref{lemExpBeta} as follows.

For Tweedie models non-linear part of dual cumulant function  $\phi_1(x)$ is 
\begin{align}
  \phi_1(x) &=    \frac{x^{2-p}}{(1-p)(2-p)}  
\end{align}
 Hence, expected \Beta divergence for Tweedie models is as
\begin{align}
  \E{d_\beta(x,\mu)} &=  \frac{ \E{x^{2-p}}  -  \mu^{2-p}}{(1-p)(2-p)} 
\end{align}
For the limits we can either use l'Hopital or simply apply non-linear part of dual cumulant function $\phi_1(x)$ as given in Example \ref{exDualCumulant1} that results to the following special case for $p=0$ and limits for $p=1,2$
\begin{align*}
      \E{d_\beta(x,\mu)} = \left\{ \begin{array}{ll}      
    \frac{1}{2} \left( \E{ x^2}  - \mu^2  \right)  &p=0  \\
     \E{ x \log x}  - \mu \log \mu      &p=1  \\
      -\E{  \log x}  +  \log \mu  &p=2\\
 \end{array}\right.
\end{align*}
where for $p=0$ it is equal to $\frac{1}{2} \sigma^2$.
\end{exam}

Clearly, expectation of beta divergence has connection to entropy, and hence using this expectation we can write entropy for dispersion models 
\begin{align}\label{entropy22}
  \HE{x}{\mu} &= -\E{\log g(x,\varphi)} + \varphi^{-1} \E{d_\beta(x,\mu)} 
  %&= -\E{\log g(x,\varphi)} + \varphi^{-1} \left\{ \E{\phi(x)} - \phi(\mu) \right\} 
\end{align} 
  
\begin{comment}

\Note{

\begin{align}
    \E{d_\beta(x,\mu)} &= \varphi  \left\{\HE{x}{\mu} + \E{\log g(x,\varphi)} \right\}  
\end{align} 

Gauss için:
$\varphi=\sigma^2$
\begin{align*}
    \E{d_\beta(x,\mu)} &= \sigma^2  \left\{\HE{x}{\mu} + \E{\log g(x,\sigma^2)} \right\}  \\
    &= \sigma^2  \left\{ \frac{1}{2} +	\frac12 \log(2 \pi  \, \sigma^2) + \E{\log (2\pi\sigma^2)^{-\frac{1}{2}}} \right\}  \\
    &= \sigma^2  \left\{	 \frac{1}{2} + \frac12 \log(2 \pi  \, \sigma^2) -\frac{1}{2} \log (2\pi\sigma^2) \right\} \\ 
        &= \frac{1}{2} \sigma^2 	 
\end{align*} 

For gamma
\begin{align}
  \E{\log g(\cdot)} = \E{ \log \frac{x^{-1} a^{a} \exp(- a) }{\Gamma(a)}} \\
  = -\E{ \log x }  + \log \frac{ a^{a} \exp(- a) }{\Gamma(a)}
\end{align}

\begin{align}
 H = 	 a  - \log b \,+\, \log \Gamma(a)  \,+\, (1 \,-\, a)\psi(a) 
\end{align}

Using $\E{ \log x } = \psi(a)-\log b$
\begin{align*}
X
  &= 
  -\E{ \log x }   + \log \frac{ a^{a} \exp(- a) }{\Gamma(a)} + \\ 
  &  	 a  - \log b + \log \Gamma(a)  + (1 - a)\psi(a)  \\
   &=  -\E{ \log x }   + a \log  a      	- \log b   + (1 - a)\psi(a)  \\
   &=  -\E{ \log x }   + a \log  a      	- \log b   + (1 - a)( \E{\log x} + \log b)  \\
    &=  -\E{ \log x }   + a \log  a      	- \log b   + \E{\log x} + \log b \\
    & -a \E{\log x} - a\log b  \\
    &=   a \log  a   -a \E{\log x} - a\log b  \\
\end{align*}
Here $\varphi=1/a$.
 \begin{align}
   X \varphi &=   \log  a   - \E{\log x} - \log b \\
   &= - \E{\log x} + \log (a/b) = - \E{\log x} + \log \mu 
 \end{align}
 }
\end{comment}   

\begin{exam}

Equation \eqref{entropy22} gives another way of computing the entropy as given in these examples. For the Gaussian, the dispersion is $\varphi=\sigma^2$ and 
\begin{align}
   \E{\log g(x,\sigma^2)} &= \E{\log (2\pi\sigma^2)^{-\frac{1}{2}}} \\
   \E{d_\beta(x,\mu)}   &= \frac{1}{2} \sigma^2
\end{align}
and plugging them  in \eqref{entropy22} we end up with the entropy for the Gaussian as
\begin{align}
  \HE{x}{\mu} =  \frac{1}{2}  \log (2\pi\sigma^2) + \frac{1}{2}
\end{align}

For the gamma the dispersion is $\varphi=1/a$, the expectation parameter is  $\mu=a/b$ and the expectation $\E{ \log x } = \psi(a)-\log b$. Identifying the rest as
\begin{align*}
  \E{\log g(\cdot)} &= \E{ \log \frac{x^{-1} a^{a} \exp(- a) }{\Gamma(a)}} \\
  &= -\E{ \log x }  + \log \frac{ a^{a} \exp(- a) }{\Gamma(a)} \\
  \E{d_\beta(\cdot,\cdot)} &= -\E{\log x} + \log \mu = -\psi(a)+\log b + \log (a/b)
\end{align*}
and plugging them all in \eqref{entropy22} we end up with the entropy for the Gamma as
\begin{align}
 \HE{x}{\mu} = 	 a  - \log b \,+\, \log \Gamma(a)  \,+\, (1 \,-\, a)\psi(a) 
\end{align}

\end{exam}

\section{Conclusion}

\begin{table}[!t]
\setlength{\extrarowheight}{1pt} 
\renewcommand{\arraystretch}{1.5}
\caption{Divergences and related quantities represented as a function of variance functions in integral forms. }
\label{tab.summary}
\centering
\begin{tabular}{l|l|l}\hline
{\bf } & {\bf Integration} & {\bf Range}  \\ \hline 
Dual cumulant  of $x$  & $\phi(x)=d_\beta(x|\mu_0) = \int_{\mu_0}^x \frac{x-t}{v(t)} \partiald t$ & $[\mu_0,x] $   \\ \hline
Canonical parameter  & $\theta(\mu)=\int_{\mu_0}^\mu \frac{1}{v(t)} \partiald t $  & $[\mu_0,\mu] $   \\ \hline
Cumulant   & $\psi\big(\theta(\mu)\big)=\int_{\mu_0}^\mu \frac{t}{v(t)} \partiald t $  & $[\mu_0,\mu] $    \\ \hline
Dual cumulant of $\mu$  & $\phi(\mu)=d_\beta(\mu|\mu_0) = \int_{\mu_0}^\mu \frac{\mu-t}{v(t)} \partiald t$ & $[\mu_0,\mu] $  \\ \hline
Beta divergence &   $ d_\beta(x,\mu) = \int_\mu^x \frac{x-t}{v(t)} \partiald t$ & $[\mu,x] $     \\ \hline 
Alpha divergence & $d_\alpha(x,\mu) = \int_\mu^x  \psi(\theta(x/t))  \partiald t$ &  $[\mu,x] $  \\ \hline
Log likelihood & $\LL{x}{\mu} = \int_{\mu_0}^\mu \frac{x-t}{v(t)} \partiald t $ & $[\mu_0,\mu] $  \\ \hline
Full log Likelihood & $\LL{x}{x} = \int_{\mu_0}^x \frac{x-t}{v(t)} \partiald t $ & $[\mu_0,x] $    \\ \hline

\end{tabular}
\end{table}

The main idea presented in this paper is one-to-one mapping between beta divergence and (half of) the statistical deviance. This simple idea has many consequences. First, density of dispersion models can be expressed as a function of beta divergence. Second,  beta divergence formulation that is linked to power functions can be generalized in the form of compact definite integrals based on variance functions of dispersion models. Third, many properties of beta divergence such as scaling, translation, expectation and even an alternative geometric interpretation can be obtained via this integral form. Even further alpha divergence can be represented similarly and its connection to beta divergence can be obtained after simple mathematical operations. Table \ref{tab.summary}
summarizes the formulations for functions presented in the paper.

\section{Appendix - Alpha Divergence}

Likewise beta divergence, another specialized divergence is alpha divergence which is  a type of the $f$-divergence \cite{amari1985}. It has strong connection to beta divergence, and hence here as we show  it enjoys similar compact integral representation as beta divergence. Before that we introduce $f$-divergence briefly. 

\subsection{$f$-Divergence}

The {\it $f$-divergences} are generalized KL divergences, and are  introduced independently by authors Csiszár \cite{csiszar1963}, Morimoto \cite{morimoto1963} and Ali \& Silvey \cite{ali1966} during 1960s. By definition, for any real valued convex function $f$  the $f$-divergence
is defined as \cite{csiszar1963}
\begin{align}
   d_f(x,\mu) =  \mu f(\frac{x}{\mu}) \qquad \text{with} \qquad f(1)  = 0
\end{align}
For the setting $x=1$, the divergence $d_f(1,\mu)$ becomes only a function of $\mu$
\begin{align}
  f^*(\mu) = \mu f(1/\mu)
\end{align}
where $f^*$ is called as Csiszár dual of the function $f$.

Likewise the Bregman divergence $f$-divergences are non-negative quantities as $d_f(x,\mu) \geq 0$ and iff  $d_f(x,x) = 0$. As a special case, Hellinger distance is a type of symmetric alpha divergence with $p=3/2$ that exhibits metric properties.   

\subsection{Alpha Divergence}

In the literature alpha divergence has many different forms \cite{amari1985,zhu19995,liese2006, cichocki11} where all are equivalent. The one that index variable aligns with Tweedie models index parameter is given in \cite{Yilmaz2012} as
\begin{align}
   d_\alpha(x,\mu) =  \frac{ x^{2-p} \mu^{p-1}}{(1-p)(2-p)}  -\frac{x}{1-p}  + \frac{\mu}{2-p}     
\end{align}
Here by changing $p=2-\alpha$ we obtain alpha divergence form given in \cite{cichocki11} whereas with $p=(\alpha+3)/2$ we obtain Amari alpha that generates another form given in \cite{amari10}. Likewise, for $p=2-\delta$, it turns to $\delta$-divergence, that is identical to alpha divergence with $\delta$ as the index parameter \cite{zhu19995}.

$f$-divergence is specialized to alpha divergence when dual cumulant function $\phi$ is used.

\begin{defn}
Let $x,\mu \in \Omega$. Alpha divergence of $x$ from $\mu$, denoted by $d_\alpha(x,\mu)$
is the $f$-divergence generated by the dual cumulant function $\phi$ induced by variance function $v(\mu)$.
\end{defn}

$f$-divergence requires that the function $f$ provides  $f(1)=0$. The dual cumulant function  $\phi$ accomplishes this by choosing the base lower bound as $\mu_0=1$ so that the function becomes
\begin{align}
  \phi(\mu) &=  \int_{1}^\mu  \frac{\mu-t}{v(t)}    \partiald t 
\end{align}
In this way, the function $\phi$  provides that $\phi(1) = 0$ as well as $\phi'(1) = 0$. 
%where this condition  is needed for normalization. 

\begin{lem}
Alpha divergence $d_\alpha(x,\mu)$  is equal to 
\begin{align}
  d_\alpha(x,\mu)  =  \mu \int_1^{x/\mu} \frac{(x/\mu)-t}{v(t)} \partiald t
\end{align} 
\end{lem}

\begin{proof}
By simply specializing the $f$-divergence definition $d_f(x,\mu) = \mu f(x/\mu)$ as
\begin{align}
  d_\alpha(x,\mu) = \mu \phi(x/\mu) 
\end{align}
\end{proof}

In the following we show well known symmetry condition for alpha divergence.

\begin{cor}
For functions that provide $\phi(r) = r \phi(1/r)$, alpha divergence is symmetric, i.e.
\begin{align}
d_\alpha(x,\mu) = d_\alpha(\mu,x)  \qquad \Rightarrow \qquad \phi(r) = r \phi(1/r)   
\end{align}

\end{cor}

\begin{proof}

\end{proof}

We want to find functions $\phi$ that generates symmetric alpha divergences such as
\begin{align}
d_\alpha(x,\mu) = d_\alpha(\mu,x)  
\end{align}
Then we plug the definition of alpha divergence as
\begin{align}
  \mu \int_1^{x/\mu} \frac{x/\mu -t}{v(t)} dt = x \int_1^{\mu/x} \frac{\mu/x -t}{v(t)} dt 
\end{align}
% SILME
%\begin{align}
%  \mu / x =  \left\{ \int_1^{\mu/x} \frac{\mu/x -t}{v(t)} dt \right\}\Big/
%  \left\{ \int_1^{x/\mu} \frac{x/\mu -t}{v(t)} dt \right\} 
%\end{align}
and then setting $r=\mu/x$ we identify the function $\phi()$
\begin{align*}
  r =  \frac{ \int_1^{r} \frac{r -t}{v(t)} dt }
  {\int_1^{1/r} \frac{1/r -t}{v(t)} dt } = 
  \frac{\phi(r)}{\phi(1/r)}   \qquad \Rightarrow \qquad \phi(r) = r \phi(1/r)
\end{align*}
 
\begin{exam}
For Tweedie models with $v(\mu)=\mu^p$ we find that 
\begin{align}
  &i) \quad \phi_p(\mu) = \mu \phi_p (1/\mu) \qquad \Rightarrow \qquad p=3/2 \\
  &ii) \quad \phi_p(\mu) = \mu \phi_q (1/\mu) \qquad \Rightarrow \qquad p+q=3 
\end{align}
\end{exam}

Alpha divergence has an interesting compact integral representation in terms of the canonical parameter and the cumulant function. In the following we first obtain integral representation of the cumulant function. 

\begin{lem}
The cumulant function is expressed in integral form  as 
\begin{align}
  \psi(\theta(\mu)) = \int_1^{\mu} \frac{t}{v(t)} \partiald t
\end{align}
\end{lem}

\begin{proof}
The cumulant can be found trivially using the duality as
\begin{align}
  \psi(\theta(\mu)) = \mu\theta(\mu) - \phi(\mu)
\end{align}
and plugging in the relevant terms we obtain the cumulant function
\begin{align}\label{alpha22}
\psi(\theta(\mu)) &= 
\mu\int_1^{\mu} \frac{1}{v(t)} \partiald t 
-  \int_1^\mu  \frac{\mu-t}{v(t)}    \partiald t 
= \int_1^{\mu} \frac{t}{v(t)} \partiald t 
\end{align}

\end{proof}
%We note that representation \eqref{alpha22}  easily obtains the initial condition as $\psi(\theta(1))=0$.

\begin{lem}

The \Alpha divergence  can also be expressed in terms of the cumulant function as
\begin{align}\label{lemAlpha2}
 d_\alpha(x,\mu) = \int_\mu^x  \psi(\theta(x/t))  \partiald t
\end{align}
\end{lem}

\begin{proof}
We indeed prove that alpha divergence can be written as in the following
\begin{align*}
    d_\alpha(x,\mu) = 
   \mu \int_1^{x/\mu} \frac{(x/\mu)-t}{v(t)} \partiald t = \int_\mu^x  \left( \int_1^{x/t} \frac{z}{v(z)} \partiald z \right) \partiald t
\end{align*}
by changing the bounds for the variables in the double integral
\begin{align*}
  1 < z < x/t &\qquad \Rightarrow \qquad  1 < z < x/\mu \\
     \mu < t < x &\qquad \Rightarrow \qquad  \mu < t < x/z
\end{align*}
and then by changing the order of the integration. Then by definition the term inside the parenthesis turns to $\psi(\theta(x/t))$. 

\end{proof}

%\Note{'\Alpha divergence for Tweedie models' ne demek?  Tweedie PVF ilişkisinden belki alpha div for PVF denebilir mi?}

\begin{exam}
We obtain \Alpha divergence for Tweedie models using directly by  \eqref{lemAlpha2}
\begin{align}
   d_\alpha(x,\mu) &= 
\int_\mu^x  \psi(\theta(x/t))  \partiald t 
   = \int_\mu^x  \frac{(x/t)^{2-p}-1}{2-p}  \partiald t 
\end{align}
where for Tweedie models $\psi(\theta(\mu))$ is
\begin{align*}
  \psi(\theta(\mu)) &= \int_1^{\mu} \frac{t}{t^p} \partiald t
 =  \frac{\mu^{2-p} -1}{2-p}   
\end{align*}

\end{exam}

\subsection{Connection of Alpha and Beta Divergences}

\begin{cor}
Alpha divergence can be written in terms of beta divergence $x/\mu$ from $1$ 
\begin{align}
  d_\alpha(x,\mu) = \mu d_\beta(x/\mu,1) 
  %= \mu \int_1^{x/\mu} \frac{(x/\mu)-t}{v(t)} \partiald t
\end{align}

\end{cor}
\begin{proof}
Integral form of \Alpha divergence trivially regarded as \Beta divergence scaled by $\mu$.
\end{proof}

The connection between \Beta and \Alpha divergences can be interpreted as Csizar's duality such that \Alpha divergence is Csizar dual of \Beta divergence implied by the definition.

\begin{cor}
Connection of alpha and beta divergences is given as
\begin{align}\label{conAlphaBeta}
  d_\beta(x,\mu) =  \frac{\mu}{f(\mu)} d_\alpha(x,\mu)
\end{align}

\end{cor}
\begin{proof}
Using the relations
\begin{align}
  i) &\quad d_\alpha(x,\mu) = \mu d_\beta(x/\mu,1)  \\
  %= \mu \int_1^{x/\mu} \frac{(x/\mu)-t}{v(t)} \partiald t 
  ii) &\quad
  d_\beta(x,\mu) =  \frac{\mu^2}{f(\mu)} d_\beta(x/\mu,1)
\end{align}
we obtain connection of alpha and beta given in \eqref{conAlphaBeta} where $f$ provides 
 \begin{align}
    f(c) =  \frac{v(ct)}{v( t)  }, \qquad f(1)=1  
\end{align}

\end{proof}

\begin{exam}
For Tweedie models with $v(\mu)=\mu^p$ the connection is
\begin{align}
  d_\beta(x,\mu) =  \mu^{1-p} d_\alpha(x,\mu)
\end{align}
since $f(\mu)=\mu^p$.
\end{exam}

\begin{cor}
Any variance function $v(\mu)=k\mu$, $k \in \mathbb{R}_+$ generates alpha and beta divergences that are equal in value 
\begin{align}
  d_\beta(x,\mu) =   d_\alpha(x,\mu) \quad \text{for} \quad v(\mu)=k\mu.
\end{align} 
\end{cor}

\begin{proof}
Alpha and beta divergences are equal when 
\begin{align}
  \frac{\mu}{f(\mu)}=1 \quad \Rightarrow \quad f(\mu)=\mu.
\end{align}
Then the variance functions turns to 
 \begin{align}
  f(c) =  \frac{v(c\mu)}{v( \mu)  }  \quad \Rightarrow \quad c =  \frac{v(c\mu)}{v( \mu)  } \quad \Rightarrow \quad v(\mu)=k\mu.
\end{align}
%for any  $k \in \mathbb{R}/\{0\}$.  
\end{proof}
For $k=1$, the variance function becomes $v(\mu)=\mu$ that corresponds to the Poisson distribution.

\bibliographystyle{IEEEtran}

% Generated by IEEEtran.bst, version: 1.13 (2008/09/30)

%\vspace{-14cm}
%\begin{IEEEbiography}[{\includegraphics[width=1in,height=1.25in,clip,keepaspectratio]{fig/kenan180x.jpg}}]%
\begin{IEEEbiography}[{\includegraphics[height=1in]{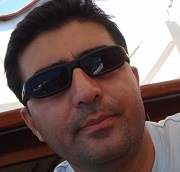}}]{Y. Kenan Y{i}lmaz} received the BS degree (1992) from the Department of Computer Engineering,  MS degree (1998) from the Institute of BioMedical Engineering and the PhD (2012) 
from the Department of Computer Engineering, all in Boğaziçi University in Istanbul. His PhD was on generalized tensor factorization. 
He is currently an IBM certified instructor for certain IBM products including Domino, Websphere Application Server and Portal Server.  
\end{IEEEbiography}

\end{document}